\documentclass[12pt]{article}

\usepackage[inline]{enumitem}
\usepackage{float}
\usepackage{mathtools}
\usepackage{graphicx}
\usepackage[round]{natbib}
\usepackage[margin=1in]{geometry}
\usepackage{tikz}
\usepackage[english]{babel}
\usepackage{longtable}
\usepackage{color}
\usepackage{amssymb,amsmath,amsthm}
\usepackage{multirow}
\usepackage[titletoc,title]{appendix}
\usepackage{authblk}
\usepackage{setspace}
\usepackage{dsfont}
\usepackage[OT1]{fontenc}
\usepackage{subcaption}
\usepackage{refcount}

\usepackage[colorlinks,citecolor=blue,urlcolor=blue]{hyperref}

\newtheorem{remark}{Remark}
\newtheorem{theorem}{Theorem}

\DeclareMathOperator{\expit}{expit}
\DeclareMathOperator{\logit}{logit}
\DeclareMathOperator{\var}{Var}

\DeclareMathOperator{\tmlee}{tmle}
\DeclareMathOperator{\aipww}{aipw}
\DeclareMathOperator{\opt}{opt}
\DeclareMathOperator{\infun}{IF}
\DeclareMathOperator{\dtmlee}{dr}
\DeclareMathOperator{\cfdtmlee}{cfdr}

\DeclareMathOperator{\R}{\mathbb{R}}

\newcommand{\indep}{\mbox{$\perp\!\!\!\perp$}}

\newcommand{\dd}{\mathrm{d}}
\newcommand{\Pn}{\mathbb{P}_{n}}

\newcommand{\hopt}{\hat l_{\opt}}
\newcommand{\aipw}{\hat\theta_{\aipww}}
\newcommand{\cfdtmle}{\hat\theta_{\cfdtmlee}}
\newcommand{\dtmle}{\hat\theta_{\dtmlee}}
\newcommand{\tmle}{\hat\theta_{\tmlee}}
\newcommand{\hsigma}{\hat\sigma_{\dtmlee}}

\newcommand{\one}{\mathds{1}}
\newcommand{\midd}{\,\bigg|\,}
{
  \theoremstyle{definition}
  \newtheorem{assumption}{}
}
{
  \theoremstyle{definition}
  \newtheorem{assumptioniden}{}
}

\newtheorem{lemma}{Lemma}

\title{Statistical Inference for Data-adaptive Doubly Robust
  Estimators with Survival Outcomes}
\date{\today}

\author[1]{Iv\'an D\'iaz \thanks{corresponding author:
    ild2005@med.cornell.edu}}
\affil[1]{\small Division of Biostatistics,
  Weill Cornell Medicine.}

\begin{document}
\maketitle

\begin{abstract}
  The consistency of doubly robust estimators relies on consistent
  estimation of at least one of two nuisance regression parameters.
  In moderate to large dimensions, the use of flexible data-adaptive
  regression estimators may aid in achieving this
  consistency. However, $n^{1/2}$-consistency of doubly robust
  estimators is not guaranteed if one of the nuisance estimators is
  inconsistent. In this paper we present a doubly robust estimator for
  survival analysis with the novel property that it converges to a
  Gaussian variable at $n^{1/2}$-rate for a large class of
  data-adaptive estimators of the nuisance parameters, under the only
  assumption that at least one of them is consistently estimated at a
  $n^{1/4}$-rate. This result is achieved through adaptation of recent
  ideas in semiparametric inference, which amount to: (i)
  Gaussianizing (i.e., making asymptotically linear) a drift term that
  arises in the asymptotic analysis of the doubly robust estimator,
  and (ii) using cross-fitting to avoid entropy conditions on the
  nuisance estimators. We present the formula of the asymptotic
  variance of the estimator, which allows computation of doubly robust
  confidence intervals and p-values. We illustrate the finite-sample
  properties of the estimator in simulation studies, and demonstrate
  its use in a phase III clinical trial for estimating the effect of a
  novel therapy for the treatment of HER2 positive breast cancer.
\end{abstract}

\maketitle

\maketitle
\section{Introduction}
\label{sec1}

Doubly robust estimation is a widely used method for the estimation of
causal effects and the analysis of missing outcome data. In survival
analysis, doubly robust estimation
\cite{Diaz2018,Moore11,Parast2014,Zhang2014,Cole2004, Xie2005,
  Rotnitzky2005} proceeds by estimating two nuisance parameters: (i)
the probability of treatment as a function of covariates and the
probability of censoring conditional on covariates (henceforth
referred to as treatment-censoring mechanism), and (ii) the
probability of an outcome conditional on covariates (henceforth
referred to as the outcome mechanism). The asymptotic properties of
the doubly robust estimator such as consistency and asymptotic
distribution thus depends on the large sample behavior of functionals
of these nuisance estimators. In low dimensional problems with
categorical covariates, nuisance parameter estimation may be carried
out using the nonparametric maximum likelihood estimator (NPMLE), and
the Delta method yields asymptotic normality of the effect
estimates. In moderate to high dimensions or with continuous
covariates, the curse of dimensionality precludes the use of the
NPMLE, making it necessary to use smoothing techniques. When
(semi)-parametric smoothing methods are used (e.g., the Cox
proportional hazards model), a simple application of the Delta method
yields $n^{1/2}$-consistency and asymptotic normality of doubly robust
estimators, provided that at least one nuisance model is correctly
specified. An influence function based approach or the bootstrap may
be used to obtain asymptotically valid estimates of the variance,
confidence intervals, and p-values. Under a moderate- to
high-dimensional regime, the functional forms posed by many
(semi)-parametric models are hardly supported by a-priori scientific
knowledge, thus yielding inconsistent nuisance and doubly robust
estimators. Data-adaptive regression methods have recently been
adopted in the missing data and causal inference literature to tackle
the problem of model misspecification \cite{vanderLaanPetersenJoffe05,
  WangBembomvanderLaan06,ridgeway2007,Bembometal08a,
  lee2010improving,van2014entering,neugebauer2016case,belloni2014,
  belloni2017program, farrell2015robust}.  Techniques such as
classification and regression trees, adaptive splines, neural
networks, $\ell_1$ regularization, support vector machines, boosting
and ensembles, etc. offer a flexibility in functional form
specification that is not available for traditional approaches such as
the Cox proportional hazards model.  However, under inconsistency of
one nuisance estimator, the large sample analysis of the resulting
data-adaptive doubly robust estimators requires empirical process
conditions which are often not verifiable. This means that the finite
sample and asymptotic distribution of cannot be established, and
standard methods for computing confidence intervals and perform
hypothesis testing (such as the bootstrap and influence function based
approaches) cannot be guaranteed to be yield correct results.

We develop a doubly robust estimator of the exposure-specific survival
curve under informative missingness and a non-randomly assigned
exposure. Our estimator is asymptotically normal under the only
assumption that the at least one of the nuisance parameters is
estimated consistently at $n^{1/4}$-rate. Our asymptotic analysis of
the estimator avoids two empirical process conditions often made in
the analysis data-adaptive doubly robust estimators: the Donsker
condition (a condition on the entropy of the model) and the condition
that a drift term defined as a functional of the nuisance estimators
is asymptotically linear and therefore asymptotically Gaussian. To our
knowledge, this is the first paper concerned with Gaussianization of
this drift term, and with providing doubly-robust asymptotic
distributions in the context of survival analysis.

Our work builds on the general framework of targeted learning
\cite{vanderLaanRose11}, and is closely related to the methods in
\cite{van2014targeted,benkeser2016doubly,diaz2017doubly}. These papers
present a series of doubly robust estimators of the mean of an outcome
from incomplete data in the setting of a cross-sectional study. We
show that these ideas are generalizable to estimation with survival
outcomes subject to informative right-censoring
censoring. Generalizing previous results to more complex data
structures has proven non-trivial, because the techniques used involve
alternative representations of the doubly robust estimating equation,
and are thus specific to the estimator and data structure
considered. Our work provides insights and building blocks that are
necessary to generalize the methods to even more complex data
structures and doubly robust estimators. To remove the Donsker
condition, which may limit the class of data-adaptive estimators
allowed, we leverage previous work on cross-fitting, originally
proposed in the context of cross-validated targeted minimum loss-based
estimation (TMLE) by \cite{zheng2011cross}, and later applied to the
estimating equation approach by \cite{chernozhukov2016double}. Our
main contribution is to present a method to Gaussianize the drift
term, a problem which cross-fitting does not solve.

Related to our methods,
\cite{belloni2014,belloni2017program,farrell2015robust} study doubly
robust estimators for cross-sectional studies in high-dimensional
settings ($p>>n$) under the assumption that the functional form of the
nuisance parameters (outcome regression and treatment mechanism) can
be approximated by a generalized linear model on a known
transformation of the covariates. They show that lasso-type methods
can be used to obtain estimators that are uniformly asymptotically
normal under consistent estimation of both the outcome regression and
the treatment mechanism. The work of
\cite{avagyan2017honest,dukes2018high} extends these methods to obtain
inference and tests that remain valid under misspecification of at
most one of the nuisance models. The methods of
\cite{van2014targeted,benkeser2016doubly,diaz2017doubly} and this
paper also achieve asymptotic normality under consistent estimation of
only one of the working nuisance parameters. Unlike
\cite{belloni2014,belloni2017program,farrell2015robust,avagyan2017honest,dukes2018high},
we do not restrict our models to the class of generalized linear
functions of the covariates, allowing for general data-adaptive
methods such as those based on regression trees, adaptive splines,
neural networks, etc.

The article is organized as follows. In Section~\ref{sec:notation} we
introduce the notation and inference problem. In Section~\ref{sec:dr}
we present existing doubly-robust estimators such as the TMLE and
augmented inverse probability weighted (AIPW), and discuss their
asymptotic properties, focusing on the definition of the ``drift
term'' generated by inconsistent estimation of the nuisance
parameters. In Section~\ref{sec:drift} we present an alternative
representation of the drift term, a result that is fundamental to the
construction of the repaired TMLE in Section~\ref{sec:tmle}. We
conclude with a numerical study and an illustrative application in
sections~\ref{sec:simula} and \ref{sec:motiva}, as well as a brief
discussion in Section~\ref{sec:discuss}.

\section{Notation}\label{sec:notation}
Let $T$ denote a discrete time-to-event outcome taking values on
$\{1,\dots,K\}$. Let $C \in \{0,\dots,K\}$ denote the censoring time
defined as the time at which the participant is last observed in the
study. Let $A\in\{0,1\}$ denote study arm assignment, and let $W$
denote a vector of baseline variables which will be used to adjust for
the confounding in treatment assignment and for informative censoring.
The observed data vector for each participant is
$O=(W,A,\Delta, \tilde T)$, where $\tilde T=\min(C,T)$, and
$\Delta = \one\{T\leq C\}$ is the indicator that the participant's
event time is observed (uncensored). Here $\one(X)$ is the indicator
variable taking value $1$ if $X$ is true and $0$ otherwise.

Equivalently, we encode a participant's data vector $O$ using the
following longitudinal data structure:
\begin{equation}
  O=(W, A, R_0, L_1,  R_1, L_2\ldots, R_{K-1}, L_K),\label{O}
\end{equation}
where $R_t = \one\{\tilde T = t, \Delta=0\}$ and
$L_t= \one\{\tilde T = t, \Delta=1\}$, for $t\in\{0,\ldots,K\}$. The
sequence $R_0, L_1, R_1, L_2\ldots, R_{K-1}, L_K$ in the above display
consists of all $0$'s until the first time that either the event is
observed or censoring occurs. In the former case $L_t=1$; otherwise
$R_t=1$. For a random variable $X$, we denote its history through time
$t$ as $\bar X_t=(X_0,\ldots,X_t)$. For a given scalar $x$, the
expression $\bar X_t=x$ denotes element-wise equality. Define the
following indicator variables for each $t$:
\[I_t=\one\{\bar R_{t-1}=0, \bar L_{t-1}=0\}, \qquad J_t=\one\{\bar
  R_{t-1}=0, \bar L_t=0\}.\] The variable $I_t$ is the indicator based
on the data through time $t-1$ that a participant is at risk of the
event being observed at time $t$. Analogously, $J_t$ is the indicator
based on the outcome data through time $t$ and censoring data before
time $t$ that a participant is at risk of censoring at time $t$. By
convention we let $J_0=1$. We assume that $O\sim P_0$, where $P_0$ is
a distribution in the non-parametric model defined as all
distributions dominated by some measure $\nu$. We assume we observe an
i.i.d. sample $O_1,\ldots,O_n$ from $P_0$, and denote $\Pn$ its
distribution function. For a function $f:O\mapsto\R$, we use the
notation $Pf=\int f\dd P$.

Define the potential outcome $T_1$ as the event time that would have
been observed had study arm assignment $A=1$ and censoring time $C=K$
been externally set with probability one. For a given time point $\tau$,
we define the counterfactual survival curve under treatment arm $A=1$
as
\[\theta_0 = P_0(T_1 > \tau).\]
We focus on estimating the survival probability for treatment arm
$A=1$, estimation of $P_0(T_0 > \tau)$, where $T_0$ as the event time
that would have been observed had study arm assignment $A=0$, may be
obtained by symmetric arguments.

Define the conditional hazard function for survival at time $t$:
\begin{equation}
  h(t,w)=P_0(L_t=1|I_t = 1, A=1, W=w),\nonumber
\end{equation}
among the population at risk at time $t$ within strata of study arm
and baseline variables. Similarly, for the censoring variable $C$,
define the censoring hazard at time $t \in \{0, \dots, K \}$:
\begin{equation}
  g_{R,0}(t,w)=P_0(R_t=1|J_t=1, A=1, W=w).\nonumber
\end{equation}
We use the notation $g_{A,0}(w)=P_0(A=1|W=w)$ and
$g_0(t,w)=(g_{A,0}(w), g_{R,0}(t,w))$.  Let $p_{W}$ denote the
marginal distribution of the baseline variables $W$.  We have added
the subscript $0$ to $p_W,g,h$ to denote the corresponding quantities
under $P_0$, and will use $p_W,g,h$ without a subscript to refer to
generic quantities associated to any $P$ in the non-parametric
model. Likewise, we use $E_0$ to denote expectation under
$P_0$. Denote the survival function for $T$ at time
$t \in \{1,\dots, \tau-1\}$ conditioned on study arm $A=1$ and
baseline variables $w$ by
\begin{equation}
  S_0(t,w)=P_0(T>t|A=1,W=w).\label{S_def}
\end{equation}
Similarly, define the following function of the censoring distribution:
\begin{equation}
  G_0(t,w)=P_0(C\geq t| A=1,W=w). \label{G_def} \end{equation}
Define the following assumptions, which are standard in the analysis
of survival data under right censoring:
\begin{assumptioniden}[Consistency]\label{ass:cons}
  $T=T_1$ in the event $A=1$;
\end{assumptioniden}
\begin{assumptioniden}[Treatment assignment randomization]\label{ass:random}
  $A$ is independent of $T_1$ conditional on $W$;
\end{assumptioniden}
\begin{assumptioniden}[Random censoring]\label{ass:randomcens}
  $C$ is independent of $T_1$ conditional on $(A=1,W)$;
\end{assumptioniden}
\begin{assumptioniden}[Positivity]\label{ass:positivity}
  $P_0\{g_{A,0}(W) > \epsilon\} = P_0\{g_{R,0}(t,W) < 1 - \epsilon\} =
  1$ for some $\epsilon>0$, for each $t\in\{0,\ldots,\tau-1 \}$
\end{assumptioniden}
We make assumptions \ref{ass:cons}-\ref{ass:positivity} throughout the
manuscript.  Assumption \ref{ass:cons} connects the potential outcomes
to the observed outcome. Assumption \ref{ass:random} holds by design
in a randomized trial. Assumption \ref{ass:randomcens}, which is
similar to that in \cite{rubin1987multiple}, holds if censoring is
random within strata of treatment and baseline variables (which we
abbreviate as ``random censoring").  Assumption \ref{ass:positivity}
states that each treatment arm has a positive probability, and that
every time point has a hazard of censoring smaller than one, within
each baseline variable stratum $W=w$ with positive density under
$P_0$.

Under assumptions \ref{ass:cons}-\ref{ass:positivity}, we have
$T \indep C | A,W$ and therefore $S_0(t,w)$ and $G_0(t,w)$ have the
following product formula representations:
\begin{align}
  S_0(t,w)&=\prod_{m=1}^t \{1-h_0(m,w)\}; \qquad
            G_0(t,w)=\prod_{m=0}^{t-1} \{1-g_{R,0}(m,w)\},
          \label{defS}
\end{align}
which leads to the following identification result:
\begin{equation*}
\theta_0 = E_0\left[S_0(\tau, W)\right] =
  E_0\left[\prod_{m=1}^{\tau} \{1-h_0(m,W)\}\right].\label{def:theta}
\end{equation*}
We sometimes use the notation $\theta(P)$ to refer to the above
parameter evaluated an any arbitrary distribution $P$ of $O$ in the
non-parametric model.

For an estimator $\hat\theta$ of $\theta_0$, we refer to
\textit{$n^{1/2}$-consistency} as the property that
$n^{1/2}(\hat\theta-\theta_0)$ is bounded in probability. We describe
the estimator as \textit{asymptotic linear} if it admits the
representation
$n^{1/2}(\hat\theta-\theta_0) = \frac{1}{\sqrt{n}}\sum_i
D(O_i)+o_P(1)$ for some function $D$. Asymptotically linear estimators
are also referred to as \textit{asymptotically normal} since the
central limit theorem yields
$n^{1/2}(\hat\theta-\theta_0)\rightsquigarrow N[0,\var\{D(O)\}]$. For
an estimator $\hat f(t,o)$ of a parameter $f_0(t,o)$, and the
$L_2(P_0)$ norm $||f||^2 = \sum_t\int f(t,o)^2\dd P_0(o)$, we refer to
$n^{1/4}$-consistency as the property that
$n^{1/4}||\hat f-f_0||=o_P(1)$. For a collection of functions
$(f_1,\ldots,f_k)$, the notation $||(f_1,\ldots,f_k)||$ is used to
denote the vector of element-wise norms.

\section{Doubly robust consistency vs doubly robust inference}\label{sec:dr}

We start by presenting the efficient influence function for estimation
of $\theta_0$ in model the non-parametric model:
\begin{equation}
  D_{\eta,\theta}(O)=-\sum_{t=1}^\tau\frac{\one(A=1)I_t}{g_A(W)G(t,W)}
  \frac{S(\tau,W)}{S(t,W)}\{L_t-h(t,W)\}  + S(\tau,W) - \theta,\label{eq:defD}
\end{equation}
where we denote $\eta=(g, h)$ and $g=(g_A,g_R)$. This efficient
influence function is a fundamental object for optimal estimation of
$\theta_0$ in the non-parametric model. First, for given estimators
$\hat h$ and $\hat g$, an estimator that solves for $\theta$ in the
estimating equation $\Pn D_{\hat \eta,\theta}=0$ is consistent if at
least one of $h_0$ or $g_0$ is estimated consistently. Second,
$\var\{D_{\eta_0,\theta_0}(O)\}$ is the efficiency bound for
estimation of $\theta_0$ in the model $\mathcal M$. Specifically,
under consistent estimation of $m_0$ and $g_0$ at a fast enough rate
(which we define below), an estimator that solves
$\Pn D_{\hat \eta,\theta}=0$ has variance smaller or equal to that of
any regular, asymptotically linear estimator of $\theta_0$ in the
non-parametric model.

The estimator constructed by directly solving for $\theta$ in (the
linear equation) $\Pn D_{\hat \eta,\theta}=0$ is often referred to as
the augmented IPW estimator, and we denote it by $\aipw$. The
augmented IPW is sometimes problematic because directly solving the
estimating equation can yield an estimate out of bounds of the
parameter space \cite{Gruber2010t}. Alternatives to repair the AIPW in
cross-sectional analyses have been discussed by \cite{Kang2007,
  Robins2007, tan2010bounded}. In this paper we work under the
targeted minimum loss based estimation (TMLE) framework of
\cite{vanderLaanRubin06, vanderLaanRose11}, which provides a general
method to construct estimators that stay within the parameter
space. In general, the TMLE of $\theta_0$ is defined as a substitution
estimator $\tmle=\theta(\tilde P)$, where $\tilde P$ is an estimate of
$P_0$ constructed such that the corresponding $\tilde \eta$ and
$\theta(\tilde P)$ solve the estimating equation
$\sum_{i=1}^n D_{\tilde \eta, \theta(\tilde
  P)}(O_i)=o_P(n^{1/2})$. The estimator $\tilde P$ is constructed by
tilting an initial estimate $\hat P$ towards a solution of the
relevant estimating equation, by means of an empirical risk minimizer
in a parametric submodel. The interested reader is referred to
\cite{Diaz2018,Moore11} for more details on the construction of a TMLE
for survival analysis. The preliminary estimator $\hat P$, or the
component $\hat \eta$ necessary to evaluate $\theta(\hat P)$, may be
obtained based on data-adaptive regression methods. In this article we
do not pursue the development of estimators of $\eta_0$, but rather
rely on estimators available in the literature. In particular, we
advocate the use of stacked regression or learning ensembles, which
poses desirable oracle guarantees \cite{vanderLaanPolleyHubbard07}.

The analysis of the properties of the $\tmle$ estimator relies on (i)
the fact that it solves the efficient influence function estimating
equation, and (ii) the consistency and smoothness of the initial
estimator $\hat \eta$. In particular, define the following conditions:
\begin{assumption}[Doubly robust consistency of $\hat \eta$]\label{ass:DR1}
  Let $||\cdot||$ denote the $L_2(P_0)$ norm defined in the notation
  section. Assume $||\hat g_A - g_{A,1}||=o_P(1)$,
  $||\hat g_R - g_{R,1}||=o_P(1)$, and $||\hat h - h_1||=o_P(1)$,
  where either $(g_{A,1}, g_{R,1}) = (g_{A,0}, g_{R,0})$, or
  $h_1=h_0$.
\end{assumption}
\begin{assumption}[Donsker]\label{ass:donsker}
  Assume the class of functions
  $\{(g_A,g_R,h):||g_A-g_{A,1}||<\delta,||g_R-g_{R,1}||<\delta,||h-h_1||<\delta\}$
  is Donsker for some $\delta > 0$.
\end{assumption}
Under conditions \ref{ass:DR1} and \ref{ass:donsker}, an application
of Theorems 5.9 and 5.31 of \cite{vanderVaart98} (see also example
2.10.10 in \cite{vanderVaartWellner96}) yields
\begin{equation}
  \tmle-\theta_0= \beta(\hat \eta) +
  (\Pn - P_0)D_{\eta_1, \theta_0} + o_P\big(n^{-1/2} + |\beta(\hat \eta)|\big),\label{eq:wh}
\end{equation}
where $\beta(\hat \eta) = P_0 D_{\hat \eta, \theta_0}$. The term
$(\Pn - P_0)D_{\eta_1, \theta_0}$ is an empirical average of mean zero
i.i.d random variables, and thus converges to a normal random variable
at $n^{1/2}$-rate. Under \ref{ass:DR1}, $\beta(\hat \eta)$ converges
to zero in probability so that $\tmle$ is consistent. However,
$n^{1/2}$-consistency of $\tmle$ requires the stronger condition that
$\beta(\hat \eta) = O_P(n^{-1/2})$. This can only be proved in general
if \emph{both} $(g_{A,1},g_{R,1})=(g_{A,0},g_{R,0})$ and $h_1=h_0$, in
which case the stronger condition $\beta(\hat \eta) = o_P(n^{-1/2})$
holds. If $\hat \eta$ is estimated within a parametric model, the
delta method yields asymptotic linearity of $\beta(\hat \eta)$, which
in turn yields asymptotic linearity of $\tmle$. However, in the doubly
robust case of \ref{ass:DR1} and under data-adaptive estimation of
$\eta_0$, an asymptotic analysis of this drift term is difficult, and
the large sample distribution of the TMLE and AIPW is generally
unknown. This means that typical doubly robust estimators are ``doubly
consistent'', but they cannot be used to obtain ``doubly robust
inference'' such as confidence intervals that remain valid under
inconsistent estimation of at most one nuisance parameter.

Our main achievement is to propose an estimation technique that
Gaussianizes $\beta(\hat \eta)$, i.e., it makes this term
asymptotically linear. Gaussianizing drift terms such as
$\beta(\hat \eta)$ has been the subject of recent literature in
targeted learning
\cite{van2014targeted,benkeser2016doubly,diaz2017doubly}. These works
develop estimation techniques for various problems in cross-sectional
studies. In the next two sections we focus on the construction of
drift-corrected estimators that endow the TMLE with a doubly robust
asymptotic distribution through Gaussianization of the drift term
$\beta(\hat \eta)$. Extensions of cross-sectional techniques to the
longitudinal setting are non-trivial, as they involve constructing
asymptotic representations of $\beta(\hat\eta)$ which are
estimable. These representations depend on sequential conditional
expectations of the efficient influence function, which for the
longitudinal case involve carefully handling the at-risk sets for each
time point. The alternative representation of $\beta(\hat \eta)$ is
achieved through representations in terms of score equations in the
non-parametric model. Doubly robust inference is thus achieved through
the construction estimators $\hat\theta$ that solve such score
equations, thereby guaranteeing that $\beta(\hat \eta)$ behaves as
Gaussian variable asymptotically. The following remark provides an
argument that Gaussianizing the drift term can also aid in reducing
the bias of TMLE estimators when both nuisance estimators are
inconsistent.
\begin{remark}[Asymptotic bias of the TMLE under inconsistency of $\hat
  \eta$]\label{remark:bias}
  Assume $\hat \eta$ converges to some $\eta_1\neq \eta_0$. Let $\theta_1$
  denote the solution to $P_0D_{\eta_1,\theta}=0$, and note that
  $D_{\eta_1,\theta_1}=D_{\eta_1,\theta_0}-\theta_1 + \theta_0$. Under
  \ref{ass:donsker}, an application of Theorem 5.31 of
  \cite{vanderVaart98} yields
  \begin{equation*}
    \tmle-\theta_1= \beta(\hat \eta) +
    (\Pn - P_0)D_{\eta_1, \theta_1} + o_P\big(n^{-1/2} + |\beta(\hat \eta)|\big).
  \end{equation*}
  Substituting $D_{\eta_1,\theta_1}=D_{\eta_1,\theta_0}-\theta_1 +
  \theta_0$ yields
  \begin{equation*}
    \tmle-\theta_0= \beta(\hat \eta) +
    (\Pn - P_0)D_{\eta_1, \theta_0} + o_P\big(n^{-1/2} + |\beta(\hat \eta)|\big).
  \end{equation*}
  The empirical process term $(\Pn - P_0)D_{\eta_1, \theta_0}$ has
  mean zero. Thus, Gaussianizing $\beta(\hat \eta)$ is expected to
  reduce the bias of $\tmle$ when $\hat \eta$ is doubly inconsistent.
\end{remark}

\section{Asymptotic representation of the drift term}\label{sec:drift}

Our proposed method to endow the TMLE with a doubly robust asymptotic
distribution relies on an asymptotic representation of the drift term
$\beta(\hat \eta)$. This parameter is then estimated using targeted
minimum loss based estimation. In Theorem~\ref{lemma:betarep} below,
we show that this drift term may be written as a sum of score score
functions that depends on the true value of an additional nuisance
parameter (defined below) and the estimator $\hat \eta$, plus a term
that approaches zero at $n^{1/2}$-rate.  The insight that allows us to
construct a TMLE with doubly robust asymptotic distribution is that
Gaussianization of the drift term amounts to tilting the estimator
$\hat \eta$ in a way such that it also targets a solution of these
score equations, thereby estimating $\beta(\hat \eta)$ at
$n^{1/2}$-rate. We introduce the following assumption regarding the
convergence rate of $\hat \eta$ to $\eta_1$:
\begin{assumption}[Consistency rate for $\hat \eta$]\label{ass:DR2}
  Assume \ref{ass:DR1}. In addition, assume that
  $||\hat g_A - g_{A,1}||=o_P(n^{-1/4})$,
  $||\hat g_R - g_{R,1}||=o_P(n^{-1/4})$, and
  $||\hat h - h_1||=o_P(n^{-1/4})$.
\end{assumption}
As discussed in the introduction, the above rate is achievable by many
data-adaptive regression algorithms such as $\ell_1$ regularization,
tree-based methods, and neural networks. In particular,
\cite{benkeser2016highly} show that a rate of $n^{-1/4 - 1/[8(d+1)]}$,
where $d$ is the dimension of $W$, is achievable under the mild
assumption that the true regression function is right-hand continuous
with left-hand limits and has variation norm bounded by a
constant. Because it is generally not possible to know a-priori which
regression algorithm will be more appropriate for a given problem,
we propose to use an ensemble learner known as the super learner
\cite{vanderLaanPolleyHubbard07}. Super learning builds a combination
of predictors in a user-given library of candidate estimators, where
the weights minimize the cross-validated risk of the resulting
combination. Super learner has been shown to have important
theoretical guarantees
\cite{vanderLaanDudoitvanderVaart06,vanderVaartDudoitvanderLaan06}
such as asymptotic equivalence to the oracle selector.

The following lemma provides a representation for the drift term in
terms of score function in the tangent space of each of the models for
$g_{A,0}$, $g_{R,0}$, and $h_0$. Such approximation is achieved
through the definition of the following univariate regression
functions. For each time point $t$ and $k$, define the time-dependent
covariates $G_g(t,w) = g_{A,1}(w)G_1(t,w)$,
$C_h(k,w)=S_1(\tau,w)/S_1(k,w)$, and define
$M(w)=\sum_{t=1}^\tau S_1(t,w)$, where $G_1$ and $S_1$ denote the
censoring and survival probabilities under the limits $g_{R,1}$ and
$h_1$ of the estimators (\ref{ass:DR1}). Define the following weighted
error functions
\begin{align}
  e_{R,0}(k, w) & = E_0\left[\frac{R_k-g_{R,1}(k,W)}{G_g(k+1, W)}\midd J_k=1, A=1,C_h(k, W) = C_h(k, w)\right],\notag\\
  e_{L,0}(t, w) &=E_0\left[C_h(t, W)\{L_t-
                  h_1(t)\}\midd
                  I_t=1, A=1, G_g(t, W) =
                  G_g(t, w)\right],\label{eq:errors}\\
  e_{A,0}(w)&=E_0\left[\frac{A-g_{A,1}(W)}{g_{A,1}(W)} \midd
              M(W) = M(w)\right].\notag
\end{align}
For each $k$, define the conditional probabilities
\begin{align}
  d_{k,0}(t, w) &= P_0\left[R_t = 1\mid J_t = 1, A = 1, C_h(k, W) =
                  C_h(k,w)\right],\notag\\
  b_{k,0}(t, w) &= P_0\left[L_t = 1\mid I_t = 1, A = 1, C_h(k, W) =
                  C_h(k,w)\right],\notag\\
  u_{k,0}(t, w) &= P_0\left[R_t = 1\mid J_t = 1, A = 1, C_g(k, W) =
                  C_g(k,w)\right],\label{eq:probs}\\
  v_{k,0}(t, w) &= P_0\left[L_t = 1\mid I_t = 1, A = 1, C_g(k, W) =
                  C_g(k,w)\right],\notag\\
  q_0(w) &= P_0\left[A=1\mid S_1(\tau, W) = S_1(\tau,w)\right],\notag
\end{align}
and the corresponding time-to-event functions
\begin{align*}
  D_{k,0}(t) = \prod_{m=0}^{t-1}\{1 -
  d_{k,0}(m)\}&,\,\quad\,B_{k,0}(t) =
                \prod_{m=1}^{t}\{1-b_{k,0}(m)\},\\
  U_{k,0}(t) = \prod_{m=0}^{t-1}\{1-u_{k,0}(m)\}&,\,\quad\,V_{k,0}(t) = \prod_{m=1}^t\{1 - v_{k,0}(m)\}.
\end{align*}
We occasionally use the notation
$\lambda_0=(e_{R,0}, e_{L,0}, e_{A,0}, d_{k,0}, b_{k,0}, u_{k,0},
v_{k,0}, q_0)$ to refer to the collection of auxiliary nuisance
parameters. We now present the asymptotic representation of the drift
term.
\begin{theorem}[Asymptotic approximation of the drift term]\label{lemma:betarep}
  Define the covariates
  \begin{align}
    H_A(w)  =& \sum_{t=1}^\tau  \frac{U_{t,0}(t,w)}{G_0(t,w)}
               \frac{V_{t,0}(t-1,w)}{g_{A,0}(w)}
               e_{L,0}(t,w),\notag\\
    H_R(k,w)  =& \frac{1}{g_{A,0}(w)G_0(k+1,w)}\left\{\sum_{t=k+1}^{\tau}
                 \frac{V_{t,0}(t-1,w)}{V_{t,0}(k,w)}
                 \frac{U_{t,0}(t,w)}{U_{t,0}(k,w)}
                 \frac{G_0(k,w)}{G_0(t,w)}e_{L,0}(t,w) \right\},\label{eq:defH}\\
    H_L(t,w)=&\frac{S_0(\tau,w)}{S_0(t,w)}\left\{\sum_{k=0}^{t-1}\frac{S_0(t-1,w)}{S_0(k,w)} \frac{B_{k,0}(k,w)}{B_{k,0}(t-1,w)}
               \frac{D_{k,0}(k,w)}{D_{k,0}(t,w)}
               e_{R,0}(k,w)\right.\notag\\
             &+\left.\frac{e_{A,0}(w)}{q_0(w)D_{t,0}(t,w)}\frac{S_0(t-1,w)}{B_{t,0}(t-1,w)}\right\}\notag
  \end{align}
  and define the following score functions:
  \begin{align*}
    D_{A, \hat g}(o) &= -H_A(w)\{a-\hat g_A(w)\},\\
    D_{R, \hat g}(o) &= -\sum_{k=0}^{\tau-1}a\,j_k H_R(k, w)
                       \{r_k - \hat g_R(k,w)\},\\
    D_{L, \hat h}(o) &= -\sum_{t=1}^\tau
                       a\,
                       i_tH_L(t,w)\{l_t-\hat
                       h(t,w)\}.
  \end{align*}Under \ref{ass:DR2} we have
  $\beta(\hat \eta)=P_0\{D_{A,\hat g} + D_{R,\hat g} + D_{L,\hat h}\} + o_P(n^{-1/2})$.
\end{theorem}
The above approximation of the drift term depends only on
$\lambda_0$. Note that $\lambda_0$ depends on $w$ only through
one-dimensional transformations which are consistently estimable at
$n^{1/4}$-rate under \ref{ass:DR2}. Thus, under condition
\ref{ass:DR2}, the parameters $\lambda_0$ can be estimated
element-wise through non-parametric smoothing techniques. The
asymptotic normality result that we present in Section~\ref{sec:tmle}
requires a consistency rate assumption for estimation of
$\lambda_0$. We now discuss two possible estimation techniques and
introduce a rate assumption that will allow us to prove asymptotic
normality.


The general method for estimating (\ref{eq:errors}) and
(\ref{eq:probs}) proceeds by obtaining estimates of the covariates and
outcomes, and then applying any non-parametric regression
technique. For instance, for a second-order kernel function $K_h$ with
bandwidth $l$ the estimator of $b_{k,0}$ is given by
\begin{equation}
  \hat b_k(t,w) = \frac{\sum_{i = 1}^nI_{t,i}\,A_i\,K_{l}\{\hat
    C_h(k, W_i) -
    \hat C_h(k, w)\}L_{t,i}}{\sum_{i = 1}^nI_{t,i}\,A_i\,K_{l}\{\hat
    C_h(k, W_i) -
    \hat C_h(k, w)\}},\label{eq:bhat}
\end{equation}
where $\hat C_h(k, w) = \hat S(\tau, w) / \hat S(k, w)$, and $\hat S$
is constructed using the preliminary estimator $\hat h$ and formula
(\ref{defS}). The optimal bandwidth $\hopt$ may be chosen using K-fold
cross-validation \cite{vanderVaartDudoitvanderLaan06}. The error
functions may be estimated analogously by plugging in estimates
$\hat \eta$ in all the quantities involved, and performing kernel
smoothing. Alternatively, other non-parametric smoothing methods may
be used for this purpose. For example, the highly adaptive lasso (HAL)
\cite{benkeser2016highly} proceeds by constructing an alternative
representation of the true function as a sum of basis functions that
grows with the sample size, and then performing $\ell_1$
regularization to select the appropriate basis functions.

The analysis of the drift-corrected estimators may be complicated due
to the two-stage estimation process whereby the covariates $C$ are
estimated and then used in a univariate smoothing technique. We
introduce the following assumption about the estimators of
$\lambda_0$, which helps us isolate this univariate non-parametric
smoothing in (\ref{eq:errors}) and (\ref{eq:probs}) from the methods
used to estimate the auxiliary covariates and outcomes in the same
expressions.

\begin{assumption}[Convergence rate for auxiliary nuisance parameter
  estimators] \label{ass:ass4} Let $\hat \lambda_0$ denote
  (\ref{eq:errors}) and (\ref{eq:probs}) with the auxiliary covariates
  $G_g$, $C_h$, and $M$ replaced by estimates $\hat C_g$,
  $\hat C_h$, and $\hat M$. For example,
  \[\hat b_{k,0}=P_0\left[L_t = 1\mid I_t = 1, A = 1, \hat C_h(k, W) =
      \hat C_h(k,w)\right].\] Assume that the smoothing method used to
  obtain $\hat \lambda$ satisfies
  $||\hat\lambda - \hat\lambda_0||=o_P(n^{-1/4})$.
\end{assumption}
Note that the above assumption is purely about the consistency of the
smoothing method used to obtain $\hat \lambda$, because the covariates
$\hat C_h$, $\hat C_h$, and $\hat M$ are the same in $\hat\lambda$ and
$\hat\lambda_0$. Non-parametric smoothing methods can be expected to
satisfy this assumption in certain situations. For example, under the
assumption that the map
$c \mapsto P_0\left[L_t = 1\mid I_t = 1, A = 1, \hat C_h(k, W) =
  c\right]$ is twice differentiable, a kernel regression estimator
with optimal bandwidth guarantees the desired convergence rate
$||\hat b-\hat b_0||=o_P(n^{-1/4})$. The HAL also achieves the desired
rate under the weaker assumption that
$c \mapsto P_0\left[L_t = 1\mid I_t = 1, A = 1, \hat C_h(k, W) =
  c\right]$ is c\`adl\`ag with bounded sectional variation norm
\cite{benkeser2016highly}. We make assumption \ref{ass:ass4} through
the remainder of the manuscript.

\begin{remark}
  A substitution estimator of the drift term may be constructed by
  plugging in all the nuisance estimates in the alternative
  representation given in Theorem~\ref{lemma:betarep}. An intuitive
  solution to the doubly robust inference problem would then be to
  subtract this term from the $\tmle$. While this makes intuitive
  sense, it does not guarantee that the resulting estimator will have
  the desired properties. The reason is that this strategy fails to
  control the term $|\beta(\hat\eta)|$ that shows up in the
  $o_P(\cdot)$ expression in (\ref{eq:wh}). The authors of
  \cite{van2014targeted,benkeser2016doubly,diaz2017doubly} also
  noticed this problem in the cross-sectional setting, a more in-depth
  explanation of the issue may be found in these references.
\end{remark}


\section{TMLE with doubly robust inference}\label{sec:tmle}

We now proceed to present an estimation method to Gaussianize the
drift term $\beta(\hat \eta)$.  As discussed in the previous section,
it is necessary to construct estimators $\hat \eta$ such that
$\beta(\hat \eta)$ is asymptotically Gaussian. Lemma
\ref{lemma:betaaslin} in the Supplementary Materials shows that, for
any $\hat\eta$,
\begin{equation}
  \label{eq:wh2}
  \beta(\hat \eta) - \hat\beta(\hat \eta) = -(\Pn-P_0)\{D_{R,\hat g} +
  D_{A,\hat g} + D_{L,\hat h}\} + o_P(n^{-1/2}),
\end{equation}
where $\hat\beta(\hat \eta)$ is constructed by plugging in estimates
of $D_{R,\hat g}$, $D_{A,\hat g}$, and $D_{L,\hat h}$ in the result of
Theorem~\ref{lemma:betarep}. In light of expression~(\ref{eq:wh}), an
asymptotically linear estimator $\hat \eta$ can be achieved through
the construction of an estimator $\tilde \eta$ that satisfies
$\hat\beta(\tilde g)=0$. In the following, this construction is based
on the fact that $D_{R,\hat g}$, $D_{A,\hat g}$, and $D_{L,\hat h}$
are score equations in the model for $g_{R,0}$, $g_{A,0}$, and $h_0$,
respectively. As a result, adding the corresponding $H$ covariates to
a logistic tilting model will tilt an initial estimator
$\hat \eta=(\hat g_A,\hat g_R, \hat h)$ towards a solution
$\tilde \eta$ of the Gaussianizing equations
$\hat \beta(\tilde \eta)=0$. Our method for solving this estimating
equations is rooted in the ideas of targeted learning
\cite{vanderLaanRose11}. Readers familiar with targeted learning will
see the similarities between the iterative procedure below and the
estimators presented, e.g., in \cite{Moore11}. As in \cite{Moore11},
use the framework of targeted learning to solve the relevant
estimating equations. Unlike \cite{Moore11}, here we are not only
interested in solving the efficient influence function estimating
equation, but also in simultaneously solving the Gaussianizing
equation $\hat\beta(\tilde \eta)=0$. In what follows we will use the
following modified representation of the data set:
\begin{equation}
  \{(t,W_i,A_i,J_{t,i}, R_{t,i},I_{t+1,i},L_{t+1,i}): t = 0,\ldots,K-1; i
  =1,\ldots,n\}.\label{longform}
\end{equation}
This data set is referred to as the long form, and the
original data set
\begin{equation}
  \{(W_i,A_i,\Delta_i,\tilde T_i): i =1,\ldots,n\}\label{shortform}
\end{equation}
is referred to as the short form. 
The proposed targeted TMLE is defined by the following algorithm:
\begin{enumerate}[label = Step~\arabic*., align=left, leftmargin=*]
\item \textit{Initial estimators.} Obtain initial estimators
  $\hat g_A$, $\hat g_R$, and $\hat h$ of $g_{A,0}$, $g_{R,0}$, and
  $h_0$. These estimators may be based on data-adaptive predictive
  methods that allow flexibility in the specification of the
  corresponding functional forms. Construct estimators $\hat e_A$,
  $\hat e_R$, $\hat e_L$ by fitting a univariate regression method
  regression as described in the previous subsection. Similarly, for
  each $k$, compute estimators of $d_{k,0}$, $d_{k,0}$, $d_{k,0}$,
  $d_{k,0}$, and $d_{k,0}$ by also running univariate regressions.
\item \textit{Compute auxiliary covariates.} For each
  subject $i$, compute the auxiliary covariates $\hat H_A(W_i)$, $\hat
  H_R(t, W_i)$, and $\hat H_L(t, W_i)$ by plugging in the estimators
  of the previous step in the definitions given in (\ref{eq:defH}). In
  addition, compute the covariate
  \[\hat Z(t, W_i)=\frac{\hat S(\tau,W_i)}{\hat g_A(W_i)\hat S(t,W_i)\hat G(t,W_i)}.\]
  The covariate $\hat Z(t, W_i)$ is fundamental to obtain an estimator
  that solves the efficient influence function estimating equation
  (see \cite{Moore11}).
\item \textit{Solve estimating equations.} Estimate the parameter
  $\epsilon = (\epsilon_A, \epsilon_R, \epsilon_L)$ in the following logistic
  tilting models $g_{A,\epsilon}$, $ g_{R,\epsilon}$, and
  $h_\epsilon$ for $g_{A,0}$, $g_{R,0}$, and $h_0$:
  \begin{align}
    \logit g_{A,\epsilon}(W_i) &= \logit \hat g_A(W_i) + \epsilon_A
                                 \hat H_A(W_i)\label{eq:submodelA}\\
    \logit g_{R,\epsilon}(t,W_i) &= \logit \hat g_R(t,W_i) + \epsilon_R
                                   \hat H_R(t,W_i).\label{eq:submodelR}\\
    \logit h_\epsilon(t,W_i) &= \logit \hat h(t,W_i) + \epsilon_{L,1}
                                   \hat H_L(t,W_i) + \epsilon_{L,2}
                               \hat Z(t, W_i).\label{eq:submodelL}
  \end{align}
  where $\logit(p)=\log\{p(1-p)^{-1}\}$. Here, $\logit \hat g_R(t,w)$,
  $\logit \hat g_A(w)$, and $\logit \hat h(w)$ are offset variables
  (i.e., variables with known parameter equal to one). The above
  parameters may be estimated by fitting standard logistic regression
  models. For example, $\epsilon_R$ is estimated through a logistic
  regression model of $R_t$ on $\hat H_R(t,W_i)$ with no intercept and
  an offset term equal to $\logit \hat g_R(t,W)$ among observations
  with $(J_t, A) = (1,1)$ in the long form dataset. Analogously,
  $\epsilon_A$ may be estimated by fitting a logistic regression model
  of $A$ on $H_A$ with no intercept and an offset term equal to
  $\logit \hat g_A(W)$ using all observations in the short form
  dataset, and $\epsilon_L$ is estimated through a logistic regression
  model of $L_t$ on $(\hat H_L(t,W_i), \hat Z(t, W_i))$ with no intercept and an offset
  term equal to $\logit \hat h(t,W)$ among observations with
  $(I_t, A) = (1,1)$. Let $\hat\epsilon$ denote these estimates.
\item \textit{Update estimators and iterate.} Define the updated
  estimators as $\hat g_R = g_{R,\hat\epsilon}$,
  $\hat g_A=g_{A,\hat \epsilon}$, and $\hat h =
  h_{\hat\epsilon}$. Repeat steps 2-4 until convergence. In practice,
  we stop the iteration once
  $\max\{|\hat\epsilon_R|, |\hat\epsilon_A|, |\hat\epsilon_L|\}< 10^{-4}n^{-3/5}$.
\item \textit{Compute IPW.} Denote the estimators in the last step of
  the iteration with $\tilde g_R$, $\tilde g_A$, and $\tilde h$. The
  drift-corrected TMLE of $\theta_0$ is defined as
  \[\dtmle = \frac{1}{n}\sum_{i=1}^n \prod_{m=1}^{\tau}
      \{1-\tilde h(m,W_i)\}.\]
\end{enumerate}

The large sample distribution of the above TMLE is given in the
following theorem:

\begin{theorem}[Asymptotic Distribution of $\dtmle$]\label{theo:dr}
  Assume \ref{ass:donsker} and \ref{ass:DR2} hold for $\tilde \eta$,
  and \ref{ass:ass4} holds for $\hat\lambda$. Then
  \[n^{1/2}(\dtmle - \theta_0)\to N(0, \sigma^2),\]where
  $\sigma^2 = \var\{\infun(O)\}$ and
  $\infun(O)=D_{\eta_1, \theta_0}(O) - D_{L,h_1}(O) - D_{R,g_1}(O) -
  D_{A,g_1}(O)$. Furthermore:
  \begin{enumerate}[label=(\roman*)]
  \item if $(g_{A,1}, g_{R,1}) = (g_{A,0}, g_{R,0})$ then
    $D_{R,g_1}(O)= D_{A,g_1}(O)=0$, and
  \item if $h_1=h_0$, then $D_{L,h_1}(O)=0$.
  \end{enumerate}
  Thus, if $\eta_1=\eta_0$ then $\infun=D_{\eta_0, \theta_0}$ and
  $\dtmle$ is efficient.
\end{theorem}

The proof of this theorem is presented in the Supplementary
Materials. Broadly, the proof proceeds as follows. First, inclusion of
the covariate $\hat Z$ guarantees that the submodel
$\{h_{\epsilon}:\epsilon\}$ generates a score which is equal to the
first term in the right hand side of (\ref{eq:defD}). This is used in
the proof to show that the estimator solves the efficient influence
function estimating equation and therefore satisfies
(\ref{eq:wh}). Then, we show that $\hat \beta(\tilde g)$ is an
asymptotically linear estimator of $\beta(\tilde g)$ with influence
function $D_{L,h_1}(O)+D_{R,g_1}(O) +
D_{A,g_1}(O)$. 
Since $\hat \beta(\tilde g)=0$, this asymptotic linearity result also
implies $|\beta(\tilde g)|=O_P(n^{-1/2})$, which according to the
discussion in the previous section is a requisite for asymptotic
linearity of $\dtmle$. An important part of this theorem is that, in
the double consistency case in which $\eta_1=\eta_0$, we have
$D_{L,h_1}(O) = D_{R,g_1}(O) = D_{A,g_1}(O)=0$, and the estimator
$\dtmle$ is asymptotically equivalent to the $\tmle$, both being
efficient. Unlike $\tmle$, the distribution of the estimator $\dtmle$
under condition \ref{ass:DR2} is known, and the variance given in the
theorem can be used to compute doubly robust standard errors and to
perform hypothesis tests. That is, the Wald-type confidence interval
$\dtmle \pm z_{\alpha} \hsigma/\sqrt{n}$, where $\hsigma^2$ is the
empirical variance of $\widehat \infun(O)$ has correct asymptotic
coverage $(1-\alpha)100\%$, whenever at least one of $\tilde g$ or
$\tilde h$ converges to their true value at the stated rate.

\subsection{Removing the Donsker Condition}


Asymptotic linearity of $\dtmle$ requires Donsker condition
\ref{ass:donsker}. This is a powerful empirical processes condition
that allows the analysis of many estimators in semi-parametric models
\cite{vanderVaart98}. However, this condition may be restrictive in
high-dimensional settings, or when the estimator of the censoring
mechanism is in a large class of function. For example functions
classes with unbounded variation are generally not Donsker, and highly
adaptive estimators such as random forests may have unbounded
variation. Fortunately, \ref{ass:donsker} may be avoided by
introducing cross-fitting into our estimation procedure. Cross-fitting
was first proposed in the context of targeted minimum loss-based
estimation in \cite{zheng2011cross}, and was subsequently applied to
estimating equations in \cite{chernozhukov2016double}.

Our cross-fitting procedure proceeds as follows. Let
${\cal V}_1,\ldots,{\cal V}_J$ denote a random partition of the index
set $\{1,\ldots,n\}$ into $J$ validation sets of approximately the
same size. That is, ${\cal V}_j\subset \{1,\ldots,n\}$;
$\bigcup_{j=1}^J {\cal V}_j = \{1,\ldots,n\}$; and
${\cal V}_j\cap {\cal V}_{j'}=\emptyset$. In addition, for each $j$,
the associated training sample is given by
${\cal T}_j=\{1,\ldots,n\}\setminus {\cal V}_j$. Denote by
$\hat \eta_{{\cal T}_j}$ the estimator of $\eta_0$ obtained by
training the corresponding prediction algorithms using only data in
the sample ${\cal T}_j$. Let also $j(i)$ denote the index of the
validation set which contains observation $i$. The cross-fitted TMLE
estimator is constructed replacing $\hat \eta$ by its cross-fitted in
steps 2 and 3 of the first iteration of the TMLE algorithm described
in Section~\ref{sec:tmle}. Let $\cfdtmle$ denote the resulting
estimator. We have the following theorem.

\begin{theorem}[Asymptotic Distribution of $\cfdtmle$]\label{theo:drcf}
  Assume \ref{ass:DR2} holds for $\tilde \eta$ and \ref{ass:ass4}
  holds for $\hat\lambda$. Then
  $n^{1/2}(\cfdtmle - \theta_0)\to N(0, \sigma^2)$, where $\sigma^2$
  is defined as in Theorem~\ref{theo:dr}.
\end{theorem}

The proof of this theorem is a straightforward adaptation of the
proofs in \cite{zheng2011cross} to our Theorem~\ref{theo:dr}. The
proof rests on the key observation that for each validation set
${\cal V}_j$, the estimators $\hat g_{A, {\cal T}_{j}}$ and
$\hat g_{R, {\cal T}_{j}}$ are fixed functions, and thus no entropy
conditions are required in the application of empirical process
results. The interested reader is encouraged to consult the original
articles \cite{zheng2011cross,chernozhukov2016double} for more details
and general proofs on cross-fitting.



\section{Numerical study}\label{sec:simula}

In this section we present the results of a simulation experiment to
illustrate the finite sample performance of statistical inference
based on the asymptotic distribution given in Theorem~\ref{theo:dr}.


We evaluate our method using a covariate vector of dimension $d=10$,
where the data generating mechanism for $g_{R,0}$, $g_{A,0}$, and
$h_0$ is sparse; and $\ell_1$ regularized logistic regression is used
to estimate these nuisance parameters. This estimator satisfies
condition~\ref{ass:DR2} (see e.g., Theorem 4.1 of
\cite{rigollet2011exponential}).

For each sample size $n\in \{400, 900, 1600, 2500, 3600, 4900\}$, we
generate 1000 datasets from a conditional distribution defined as
follows. First, a covariate vector $W$ was generated from
$\text{TN}(0,\Sigma)$ where $\text{TN}$ is a multivariate normal
distribution with each margin truncated at $(-1.5, 1.5)$, and $\Sigma$
is a $10\times 10$ symmetric Toeplitz matrix with first row equal to
$(10, \ldots, 1) / 10$. We then define the unobserved variables
\begin{align*}
  U_1 &= |W_1W_2|^{1/2}-|W_{10}|^{1/2} +
        \cos(W_{5})-\cos(W_6)\cos(W_5)\\
  U_2 &= |W_1W_{10}|^{1/2}-|W_{9}|^{1/2} +
        \cos(W_{5})-\cos(W_7)\cos(W_6).
\end{align*}
The data are generated as
\begin{align*}
  A\mid W=w &\sim \text{Ber}\{g_{A,0}(u_1)\}\notag\\
  R_t\mid J_t=1,A=a,W=w &\sim \text{Ber}\{g_{R,0}(t,a,u_2)\}\\
  L_t\mid I_t=1,A=a,W=w &\sim \text{Ber}\{h_0(t,a,u_1)\}\notag,
\end{align*}
where Ber$(p)$ denotes the Bernoulli distribution with parameter $p$
and
\begin{align*}
  g_{A,0}(u) &= \expit(-2u)\notag\\
  g_{R,0}(t,a,u) &= \expit\left\{-4 + a + a\cos(t) - aut^{1/2}\right\}\\
  h_0(t,a,u) &= \expit\left\{-3 + a - 2u\log(t) + 0.5au - 0.6(a+1)u\sin(t)\right\}\notag,
\end{align*}
As previously discussed, $g_{R,0}$, $g_{A,0}$, and $h_0$ are estimated
through $\ell_1$ regularized logistic regression. For consistent
estimation of $g_{A,0}$ the design matrix contains all $W$ covariates
in addition to their absolute squared root and cosine transformations
as well as all two-way interactions between all these terms. For
consistent estimation of $g_{R,0}$ and $h_0$, the design matrix is
constructed by considering all main effects and interactions of: (i)
time as a categorical variable, (ii) the treatment indicator $A$, and
(iii) all the terms considered for $g_{A,0}$. Inconsistent estimators
were obtained through standard logistic regression with main terms
only. We considered three scenarios for estimation of the nuisance
parameters: (a) all $g_{A,0}$, $g_{R,0}$, and $h_0$ consistently
estimated, (b) only $h_0$ consistently estimated, and (c) only
$g_{A,0}$ and $g_{R,0}$ consistently estimated. We also performed a
simulation where all nuisance parameters are inconsistently estimated,
but the results are uninformative and are not presented.

We compute two estimators: a doubly robust $\tmle$ \cite{Moore11} and
our proposed $\dtmle$. The $\tmle$ estimator has been shown to
outperform the $\aipw$ estimator at finite samples in simulation
studies \cite{Porter2011}, and both are expected to have similar
asymptotic behavior. We evaluate the performance of the estimators in
terms of bias, variance, mean squared error, and coverage of the
$90\%$, $95\%$, and $99\%$ confidence intervals. Some of these
quantities are multiplied by $n^{1/2}$ to evaluate
$n^{1/2}$-consistency. We also evaluate $\hsigma$ as an estimator
of the standard error of the estimators.

The results are presented in Figure~\ref{fig:uni}. Some expected
properties of the estimators, which we corroborate in the simulation
study are:
\begin{itemize}
  \itemsep0em
\item The best behavior in terms of all metrics is obtained in
  scenario (a) for both estimators. In this case, both estimators have
  very similar asymptotic performance, with $\dtmle$ having slightly
  better bias in the smaller sample sizes.
\item $\dtmle$ has significantly smaller bias than $\tmle$ for
  scenarios (b) and (c).
\item The proposed estimator of the standard error $\hsigma$ seems to
  consistently estimate the standard error of $\dtmle$ in all three
  scenarios, whereas the na\"ive estimator for $\tmle$ seems to be
  inconsistent in scenarios (b) and (c).
\item The coverage probabilities for $\dtmle$ are closer to the
  nominal level for all sample sizes and all three scenarios. Of
  particular relevance, $\dtmle$ seems to provide very important
  small-sample gains in scenarios (b) and (c).
\end{itemize}

According to Remark~\ref{remark:bias}, solving the debiasing equation
$\beta(\hat \eta)=0$ could reduce this bias of $\dtmle$, in comparison
to the bias of $\tmle$ in the case of double inconsistency (results
not shown). However, the MSE of both estimators was identical, and
increased linearly in $n^{1/2}$-scale. Identifying scenarios under
which which this bias reduction can be expected is an open problem.

\begin{figure}[!htb]
  \centering
    \includegraphics[scale=0.5]{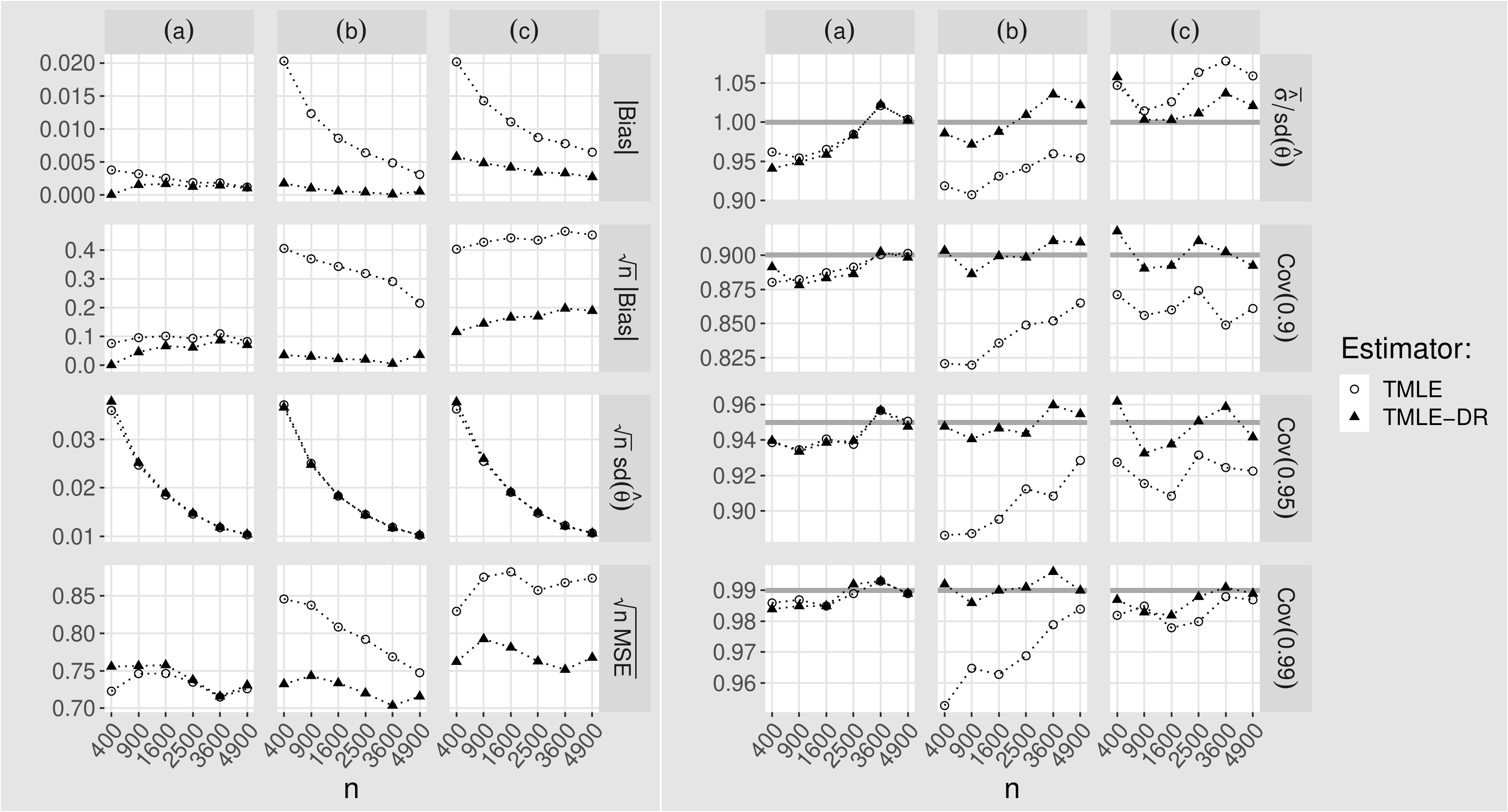}
    \caption{Results of the numerical study. The scenarios under study
      are: (a) all $g_{A,0}$, $g_{R,0}$, and $h_0$ consistently
      estimated, (b) only $h_0$ consistently estimated, and (c) only
      $g_{A,0}$ and $g_{R,0}$ consistently estimated. Cov(p) stands
      for coverage of a Wald-type 100p\% confidence interval, and
      $\text{sd}(\hat\theta)$ stands for the standard deviation of the
      estimator $\hat\theta$.}
  \label{fig:uni}
\end{figure}

\section{Motivating Application}\label{sec:motiva}
Different types of human breast cancer tumors have been shown to have
heterogeneous response to treatments \cite{perou2000, sotiriou2009}.
Amplification of ERBB2 gene and associated overexpression of human
epidermal growth factor receptor (HER2) encoded by this gene occur in
25-30$\%$ of breast cancers \cite{slamon2001use}. HER2-positive
breast cancer is an aggressive form of the disease and the prognosis
for such patients is generally poor \cite{slamon1987human,
  seshadri1993}. The clinical efficacy of adjuvant trastuzumab, a
recombinant monoclonal antibody, in early stage HER2-positive patients
was demonstrated by several large clinical trials
\cite{perez2011four, romond2005}. We illustrate our methods using
data for 1390 patients from the North Central Cancer Treatment Group
N9831 study, a phase III randomized clinical trial testing the
addition of trastuzumab to chemotherapy in stage I-III HER2-positive
breast cancer. Recruitment started in 2000, and the maximum follow-up
time was 16 years. The treatment group comprised 907 patients. The
trial was subject to right censoring because patients drop out of the
study and because enrollment spanned several years. We adjusted for 12
baseline variables which included demographic variables such as age,
ethnicity, and race; as well as clinical variables such as tumor
grade, nodal status, tumor size, and histology.

We estimated the treatment probabilities as well as the hazard of the
event and censoring using an ensemble predictor known as the super
learner \cite{vanderLaanPolleyHubbard07}, implemented in the R package
SuperLearner. Super learning builds a convex combination of candidate
predictors in a user-given library, where the weights are chosen to
minimize the cross-validated log-likelihood of the resulting
ensemble. We present the results of the ensemble in Table
\ref{tab:sl}, which includes some of the most popular statistical
learning algorithms. The tuning parameters of each algorithm are
chosen using internal cross-validation. In order to fully account for
treatment-covariate interactions, we fitted separate models for the
censoring probabilities in the treated and control arms.

\begin{table}[ht]
  \caption{Super learning ensemble coefficients. RF: random forests, XGB: extreme gradient
    boosting, MLP: multi-layer perceptron, GLM: logistic regression,
    MARS: multivariate adaptive splines, Lasso: $L_1$ regularized
    logistic regression.}
  \label{tab:sl}
  \centering
  \begin{tabular}{rcccccc}\hline
    & RF & XGB & MLP & GLM & MARS & Lasso \\
    \hline
    $g_A$        & 0.00 & 0.47 & 0.15 & 0.00 & 0.09 & 0.28 \\
    $g_R$, $A=1$ & 0.00 & 0.29 & 0.00 & 0.00 & 0.49 & 0.21 \\
    $g_R$, $A=0$ & 0.00 & 0.41 & 0.00 & 0.00 & 0.09 & 0.50 \\
    $h$,  $A=1$  & 0.17 & 0.00 & 0.00 & 0.27 & 0.36 & 0.20 \\
    $h$, $A=0$   & 0.27 & 0.13 & 0.00 & 0.00 & 0.42 & 0.18 \\
    \hline
  \end{tabular}
\end{table}

We computed the $\dtmle$ and $\tmle$ estimators separately for the
treated and untreated groups at time $\tau=12$ years. We obtained an
estimated difference ($A=1$ vs $A=0$) in survival probability of
$\dtmle = 0.107$ (s.e. $0.036$) and $0.098$ (s.e. $0.032$) years in
the treatment arm for each estimator, respectively. The Kaplan-Meier
estimator is equal to $0.044$ (s.e. $0.032$), highlighting the
possible bias due to informative censoring.

\section{Discussion}\label{sec:discuss}

Our method assumes that censoring is confounded with the time to event
only by baseline variables. In the presence of time dependent
confounding between censoring and the event time, our proposal may be
adapted by augmenting the censoring and outcome models to include
time-varying confounders. In these settings. it may be possible to
retain the asymptotic linearity to proper control of the drift term by
generalizing the techniques used to prove our theorems. Such
generalizations are unfortunately not trivial because the
representations of the drift term and therefore its targeting
algorithm varies with the estimating equation considered.

Our methods for doubly robust estimators trivially translate into
$n^{1/2}$-consistency rates for inverse probability weighted
estimators under $n^{1/4}$-consistent estimation of the $g$ components
of the nuisance parameter $\eta$. Specifically, such IPW estimator may
be obtained through our proposal by setting $\hat h(t,a,w)=1$, and
omitting the tilting model (\ref{eq:submodelL}) in the iterative
procedure that defines the proposed estimator.

Most clinical research studies use discrete time scales to measure the
time to event. This is the case of our application and simulation
studies. If time is measured on a continuous scale, implementation of
our methods requires discretization. The specific choice of the
discretization intervals may be guided by what is clinically
relevant. For example, in clinical applications with time to death
outcomes, the clinically relevant scale would typically be a day. In
the absence of clinical criteria to guide the choice of discretization
level, a concern is that too coarse of a discretization may lead to
potentially meaningful information losses. A question for future
research is how to optimally set the level of discretization in order
to trade off information loss versus estimator precision. Another area
for future research is to consider discretization levels that get
finer with sample size.

Existing doubly robust estimators cannot be proved regular or
$n^{1/2}$-consistent in general under inconsistent estimation of one
of the nuisance parameters. While we do not tackle the regularity
problem, we do solve the $n^{1/2}$-consistency problem. This is done
by proving a doubly robust asymptotic linearity result for our
estimator, under the only assumption that at least one of the nuisance
estimators is consistent at $n^{1/4}$-rate. The regularity of our
estimator remains an open problem along with that of all doubly robust
estimators based on data-adaptive estimation of nuisance parameters
under inconsistency of at least one nuisance estimator.

The $n^{1/4}$-rate required by our estimators may still be considered
a restrictive assumption. However, this rate is achievable by several
data-adaptive regression algorithms under certain assumptions on the
true regression functions. See for example
\cite{bickel2009simultaneous} for results on $\ell_1$ regularization,
\cite{wager2015adaptive} for results on regression trees, and
\cite{chen1999improved} for neural networks. This convergence rate is
also achievable by the highly adaptive lasso \cite{benkeser2016highly}
under the mild assumption that the true regression function is
right-hand continuous with left-hand limits and has variation norm
bounded by a constant.


\section{Software}
\label{sec5}

Software in the form of R code, together with a sample input data set
and complete documentation is available at
\texttt{https://github.com/idiazst/survdr}.

\section*{Supplementary Material}
\label{sec6}
\subsection*{Theorem~\ref{lemma:betarep}}

\begin{proof}
  For notational
  simplicity, in this proof we omit the dependence of all functions on
  $w$. E.g., $h_0(t,w)$ is denoted with $h_0(t)$. Lemma~1 in the
  Supplementary materials of \cite{diaz2018targeted} shows
  \begin{equation*}
    P_0 D_{\hat\eta,\theta_0}=\sum_{t=1}^\tau
    E_0\left[-\frac{\hat S(\tau)}{\hat g_A\hat S(t)\hat G(t)}S_0(t-1)\{h_0(t)-\hat
      h(t)\}\left\{g_{A,0} G_0(t)- \hat g_A\hat G(t)\right\}\right].
  \end{equation*}
  First, note that
  \begin{align}
    \{h_0(t) - \hat h(t)\}\{g_{A,0}G_0(t) - \hat g_{A}\hat G(t)\}
    & = \{h_0(t) - h_1(t)\}\{g_{A,0}G_0(t) - \hat g_A\hat G(t)\}\notag\\
    & + \{h_0(t) - \hat h(t)\}\{g_{A,0}G_0(t) - g_{A,1}G_1(t)\}\notag\\
    & + \{h_0(t) - h_1(t)\}\{g_{A,0}G_0(t) - g_{A,1}G_1(t)\}\label{zero}\\
    & + \{h_1(t) - \hat h(t)\}\{g_{A,1}G_1(t) - \hat g_{A}\hat G(t)\}.\label{onhalf}
  \end{align}
  By assumption, the expectation of (\ref{zero}) with respect to $P_0$
  is zero, and the expectation of (\ref{onhalf}) is $o_P(n^{-1/2})$. Define
  \begin{align*}
    \beta_g(\hat g)&=\sum_{t=1}^\tau
                     E_0\left[-\frac{\hat S(\tau)}{\hat g_A\hat S(t)\hat G(t)}S_0(t-1)\{h_0(t)-
                     h_1(t)\}\left\{g_{A,0} G_0(t)- \hat g_A\hat
                     G(t)\right\}\right],\\
    \beta_h(\hat h)&=\sum_{t=1}^\tau
                     E_0\left[-\frac{\hat S(\tau)}{\hat g_A\hat S(t)\hat G(t)}S_0(t-1)\{h_0(t)-
                     \hat h(t)\}\left\{g_{A,0} G_0(t)- g_{A,1}                     G_1(t)\right\}\right]\\
  \end{align*}
  Assume first that $g_1=g_0$.  Denote
  \begin{align*}
    \hat e_{R,0}(k, w) & =
                         E_0\left[\frac{R_k-g_{R,1}(k,W)}{g_{A,1}(W)G_1(k+1,W)}\midd
                         J_k=1, A=1,\hat C_h(k, W) = \hat C_h(k, w)\right]\\
    \hat e_{L,0}(t, w) &=E_0\left[\frac{S_1(\tau)}{S_1(t)}\{L_t-
                         h_1(t)\}\midd
                         I_t=1, A=1, \hat C_g(t, W) =
                         \hat C_g(t, w)\right],
  \end{align*}
  where the expectation is taken with respect to $P_0$ taking
  $\hat C_h$ and $\hat C_g$ as fixed functions. We have
{\small  \begin{align}
    \beta_g(\hat g)
    =&\sum_{t=1}^\tau
       E_0\left[-\frac{S_1(\tau)}{S_1(t)}S_0(t-1)\{h_0(t)-h_1(t)\} \left\{\frac{g_{A,0}}{\hat
       g_A}\frac{G_0(t)}{\hat G(t)}-1\right\}\right]+o_P(n^{-1/2})\notag\\
    =&\sum_{t=1}^\tau
       E_0\left[-\frac{S_1(\tau)}{S_1(t)}AI_t\{L_t-
       h_1(t)\}\left\{\frac{1}{\hat g_A \hat G(t)} - \frac{1}{g_{A,0}
       G_0(t)}\right\}\right]+o_P(n^{-1/2})\notag\\
    =&\sum_{t=1}^\tau
       E_0\left[-AI_t e_{L,0}(t)\left\{\frac{1}{\hat g_A \hat G(t)} - \frac{1}{g_{A,0}
       G_0(t)}\right\}\right]- \sum_{t=1}^\tau E_0\left[\frac{AI_t}{\hat g_A \hat G(t)}
       \{\hat e_{L,0}(t) - e_{L,0}(t)\}\right] +o_P(n^{-1/2})\notag\\
    =&\sum_{t=1}^\tau
       E_0\left[-AI_t e_{L,0}(t)\left\{\frac{1}{\hat g_A \hat G(t)} - \frac{1}{g_{A,0}
       G_0(t)}\right\}\right]+o_P(n^{-1/2})\notag\\
    =&\sum_{t=1}^\tau
       E_0\left[-AV_{t,0}(t-1)U_{t,0}(t) e_{L,0}(t)\left\{ \frac{1}{\hat g_A \hat G(t)} - \frac{1}{g_{A,0}
       G_0(t)}\right\}\right]+o_P(n^{-1/2})\notag\\
    =&\sum_{t=1}^\tau
       E_0\left[-AV_{t,0}(t-1)U_{t,0}(t)
       \frac{e_{L,0}(t)}{g_{A,0}G_0(t)}\left\{\frac{g_{A,0}}{\hat
       g_A\hat G(t)}\{G_0(t)-\hat G(t)\}
       + \frac{1}{\hat g_A}(g_{A,0}-\hat g_A)\right\}\right]+o_P(n^{-1/2})\notag\\
    =&\sum_{t=1}^\tau
       E_0\left[-AV_{t,0}(t-1)U_{t,0}(t)
       \frac{e_{L,0}(t)}{g_{A,0}G_0(t)}\left\{\frac{g_{A,0}}{\hat
       g_A\hat G(t)}\{G_0(t)-\hat G(t)\}\right\}\right]\label{eq:term1}\\
     &- \sum_{t=1}^\tau
       E_0\left[V_{t,0}(t-1)U_{t,0}(t)
       \frac{e_{L,0}(t)}{G_0(t)}\frac{1}{g_{A,0}}(A-\hat g_A)\right]+o_P(n^{-1/2})\label{eq:term2},
  \end{align}}
  where we get
  \[\sum_{t=1}^\tau E_0\left[\frac{AI_t}{\hat g_A \hat G(t)}
      \{\hat e_{L,0}(t) - e_{L,0}(t)\}\right]=0\] using the law of iterated
  expectation. The term (\ref{eq:term2}) is in the desired
  form. It remains to prove the result for (\ref{eq:term1}). Define
  \begin{align*}
    M_0(k, w)& = \{g_{R,0}(k,w)
               - \hat g_R(k,w)\}\frac{G_0(k,w)}{G_0(k+1,w)},\\
    \tilde u_{k,0}(t, w) &= P_0\left[R_t = 1\mid J_t = 1, A = 1, C_g(k, W) =
                           C_g(k,w),M_0(k, W)=M_0(k, w)\right],\\
    \tilde v_{k,0}(t, w) &= P_0\left[L_t = 1\mid I_t = 1, A = 1, C_g(k, W) =
                           C_g(k,w),M_0(k, W)=M_0(k, w)\right],
  \end{align*}
  and notice that
  $P_0\{\tilde v_{k,0}(t) - v_{k,0}(t)\}=O_P(||\hat g_R - g_{R,0}||)$,
  $P_0\{\tilde u_{k,0}(t) - u_{k,0}(t)\}=O_P(||\hat g_R - g_{R,0}||)$.
  Then (\ref{eq:term1}) is equal to
{\footnotesize  \begin{align*}
    \sum_{t=1}^\tau
    E_0&\left[-AV_{t,0}(t-1)U_{t,0}(t)
         \frac{e_{L,0}(t)}{g_{A,0}G_0(t)}\left\{\frac{g_{A,0}}{\hat
         g_A\hat G(t)}\{G_0(t)-\hat G(t)\}\right\}\right]\\
       &= \sum_{t=1}^\tau E_0\left[-AV_{t,0}(t-1)U_{t,0}(t)
         \frac{e_{L,0}(t)}{g_{A,0}G_0(t)}\left\{\frac{g_{A,0}}{\hat
         g_A\hat G(t)}\sum_{k=0}^{t-1}G_0(k)\{g_{R,0}(k)
         - \hat g_R(k)\}\frac{\hat G(t)}{\hat G(k+1)}\right\}\right]\\
       &= \sum_{t=1}^\tau\sum_{k=0}^{t-1}E_0\left[-\frac{A}{g_{A,0}}V_{t,0}(t-1)U_{t,0}(t)
         \frac{e_{L,0}(t)}{G_0(t)}\left\{\{g_{R,0}(k)
         - \hat g_R(k)\}\frac{G_0(k)}{G_0(k+1)}\right\}\right]\\
       &= \sum_{t=1}^\tau\sum_{k=0}^{t-1}E_0\left[-\frac{A}{g_{A,0}}V_{t,0}(t-1)U_{t,0}(t)
         \frac{e_{L,0}(t)}{G_0(t)}\frac{1-R_0}{1-\tilde u_{t,0}(0)}\left\{\{g_{R,0}(k)
         - \hat g_R(k)\}\frac{G_0(k)}{G_0(k+1)}\right\}\right]\\
       &= \sum_{t=1}^\tau\sum_{k=0}^{t-1}E_0\left[-\frac{A}{g_{A,0}}V_{t,0}(t-1)U_{t,0}(t)
         \frac{e_{L,0}(t)}{G_0(t)}\frac{1-R_0}{1-u_{t,0}(0)}\left\{\{g_{R,0}(k)
         - \hat g_R(k)\}\frac{G_0(k)}{G_0(k+1)}\right\}\right] +
         O_P(||\hat g_R - g_{R,0}||^2)\\
       &=\sum_{t=1}^\tau\sum_{k=0}^{t-1}E_0\left[-\frac{A}{g_{A,0}}V_{t,0}(t-1)U_{t,0}(t)
         \frac{e_{L,0}(t)}{G_0(t)}\frac{1-R_0}{1-u_{t,0}(0)}\left\{\{g_{R,0}(k)
         - \hat g_R(k)\}\frac{G_0(k)}{G_0(k+1)}\right\}\right] +
         o_P(n^{-1/2})\\
       &= \sum_{t=1}^\tau\sum_{k=0}^{t-1}E_0\left[-\frac{A}{g_{A,0}}V_{t,0}(t-1)U_{t,0}(t)
         \frac{e_{L,0}(t)}{G_0(t)}\frac{1-R_0}{1-u_{t,0}(0)}\frac{1-L_1}{1-v_{t,0}(1)}
         \left\{\{g_{R,0}(k)
         - \hat g_R(k)\}\frac{G_0(k)}{G_0(k+1)}\right\}\right] +
         o_P(n^{-1/2})\\
       &=
         \sum_{t=1}^\tau\sum_{k=0}^{t-1}E_0\left[-\frac{A}{g_{A,0}}\frac{V_{t,0}(t-1)}{V_{t,0}(k)}
         \frac{U_{t,0}(t)}{U_{t,0}(k)}\frac{e_{L,0}(t)}{G_0(t)}J_k
         \left\{\{g_{R,0}(k)
         - \hat g_R(k)\}\frac{G_0(k)}{G_0(k+1)}\right\}\right] + o_P(n^{-1/2})\\
       &=
         \sum_{t=1}^\tau\sum_{k=0}^{t-1}E_0\left[-\frac{A}{g_{A,0}}\frac{V_{t,0}(t-1)}{V_{t,0}(k)}
         \frac{U_{t,0}(t)}{U_{t,0}(k)}\frac{e_{L,0}(t)}{G_0(t)}J_k
         \left\{\{R_k
         - \hat g_R(k)\}\frac{G_0(k)}{G_0(k+1)}\right\}\right] +
         o_P(n^{-1/2})\\
       &=
         \sum_{k=0}^{\tau-1}\sum_{t=k+1}^{\tau}E_0\left[-\frac{A}{g_{A,0}}\frac{V_{t,0}(t-1)}{V_{t,0}(k)}
         \frac{U_{t,0}(t)}{U_{t,0}(k)}\frac{e_{L,0}(t)}{G_0(t)}J_k
         \left\{\{R_k
         - \hat g_R(k)\}\frac{G_0(k)}{G_0(k+1)}\right\}\right] + o_P(n^{-1/2}),
  \end{align*}}
  where the first equality follows from Lemma~\ref{lemma:telescope}.

  Assume now $h_1=h_0$. Then
  \begin{align}
    \beta_h(\hat h)  =&\sum_{t=1}^\tau E_0\left[-\frac{\hat S(\tau)}{\hat
                        S(t)}S_0(t-1)\{h_0(t)-\hat h(t)\}\left\{\frac{g_{A,0}}{g_{A,1}}\sum_{k=0}^{t-1}\{g_{R,0}(k) - g_{R,1}(k)\}\frac{G_0(k)}{G_1(k+1)}\right\}\right]\label{betaR}\\
    +&\sum_{t=1}^\tau
       E_0\left[-\frac{\hat S(\tau)}{\hat S(t)}S_0(t-1)\{h_0(t)-\hat h(t)\}\frac{1}{g_{A,1}}(g_{A,0}-g_{A,1})\right].\label{betaA}
  \end{align}
  We first tackle the term in (\ref{betaR}). This term is equal to
  \[E_0\left[\frac{g_{A,0}}{g_{A,1}}\sum_{k=0}^{\tau-1}\{g_{R,0}(k) - g_{R,1}(k)\}\frac{G_0(k)}{G_1(k+1)}\sum_{t=k+1}^\tau \frac{\hat S(\tau)}{\hat
        S(t)}S_0(t-1)\{\hat h(t)- h_0(t)\}\right].\]
  We have
  \begin{align*}
    \sum_{t=k+1}^\tau \frac{\hat S(\tau)}{\hat
    S(t)}S_0(t-1)\{\hat h(t)-\hat h_0(t)\}&=\sum_{t=1}^\tau \frac{\hat S(\tau)}{\hat
                                            S(t)}S_0(t-1)\{\hat h(t)- h_0(t)\}
                                            - \sum_{t=1}^k\frac{\hat S(\tau)}{\hat
                                            S(t)}S_0(t-1)\{\hat
                                            h(t)-h_0(t)\}\\
                                          &=S_0(\tau) - \hat S(\tau) -\frac{\hat S(\tau)}{\hat S(k)}\{S_0(k) - \hat
                                            S(k)\}\\
                                          &=S_0(k)\left\{\frac{\hat S(\tau)}{\hat S(k)}-\frac{S_0(\tau)}{S_0(k)}\right\},\\
  \end{align*}
  where the second equality follows from
  Lemma~\ref{lemma:telescope}. Thus, (\ref{betaR}) equals
  \begin{align*}
    E_0&\left[\sum_{k=0}^{\tau-1}\left\{\frac{g_{A,0}}{g_{A,1}}G_0(k)S_0(k)\{g_{R,0}(k) - g_{R,1}(k)\}\frac{1}{G_1(k+1)}\right\}\left\{\frac{\hat S(\tau)}{\hat S(k)}-\frac{S_0(\tau)}{S_0(k)}\right\}\right]\\
       &=E_0\left[\sum_{k=0}^{\tau-1}\left\{\frac{A
         J_k}{g_{A,1}}\{R_k -
         g_{R,1}(k)\}\frac{1}{G_1(k+1)}\right\}\left\{\frac{\hat
         S(\tau)}{\hat
         S(k)}-\frac{S_0(\tau)}{S_0(k)}\right\}\right]\\
       &=E_0\left[\sum_{k=0}^{\tau-1}A J_k
         e_{R,0}(k)\left\{\frac{\hat S(\tau)}{\hat
         S(k)}-\frac{S_0(\tau)}{S_0(k)}\right\}\right] +
         E_0\left[\sum_{k=0}^{\tau-1}A J_k\frac{\hat S(\tau)}{\hat
         S(k)}\{\hat e_{R,0}(k) -
         e_{R,0}(k)\}\right] \\
       &=E_0\left[\sum_{k=0}^{\tau-1}A \one\{\bar R_{k-1}=0,\bar L_k=0\}
         e_{R,0}(k)\left\{\frac{\hat S(\tau)}{\hat
         S(k)}-\frac{S_0(\tau)}{S_0(k)}\right\}\right]\\
       &=E_0\left[\sum_{k=0}^{\tau-1}A \one\{\bar R_{k-1}=0,\bar L_{k-1}=0\}(1-L_k)
         e_{R,0}(k)\left\{\frac{\hat S(\tau)}{\hat
         S(k)}-\frac{S_0(\tau)}{S_0(k)}\right\}\right]\\
       &=E_0\left[\sum_{k=0}^{\tau-1}A \one\{\bar R_{k-1}=0,\bar L_{k-1}=0\}\{1-b_{k,0}(k)\}
         e_{R,0}(k)\left\{\frac{\hat S(\tau)}{\hat
         S(k)}-\frac{S_0(\tau)}{S_0(k)}\right\}\right]\\
       &=E_0\left[\sum_{k=0}^{\tau-1}A \one\{\bar R_{k-2}=0,\bar L_{k-1}=0\}(1-R_{k-1})\{1-b_{k,0}(k)\}
         e_{R,0}(k)\left\{\frac{\hat S(\tau)}{\hat
         S(k)}-\frac{S_0(\tau)}{S_0(k)}\right\}\right]\\
       &=E_0\left[\sum_{k=0}^{\tau-1}A \one\{\bar R_{k-2}=0,\bar L_{k-1}=0\}\{1-d_{k,0}(k-1)\}\{1-b_{k,0}(k)\}
         e_{R,0}(k)\left\{\frac{\hat S(\tau)}{\hat
         S(k)}-\frac{S_0(\tau)}{S_0(k)}\right\}\right]\\
       &=E_0\left[\sum_{k=0}^{\tau-1}A \one\{\bar R_{k-2}=0,\bar
         L_{k-1}=0\}\frac{D_{k,0}(k)}{D_{k,0}(k-1)}\frac{B_{k,0}(k)}{B_{k,0}(k-1)}e_{R,0}(k)\left\{\frac{\hat S(\tau)}{\hat S(k)}-\frac{S_0(\tau)}{S_0(k)}\right\}\right]\\
       &=E_0\left[\sum_{k=0}^{\tau-1}A \one\{\bar R_{k-3}=0,\bar
         L_{k-2}=0\}\frac{D_{k,0}(k)}{D_{k,0}(k-2)}\frac{B_{k,0}(k)}{B_{k,0}(k-2)}e_{R,0}(k)\left\{\frac{\hat
         S(\tau)}{\hat S(k)}-\frac{S_0(\tau)}{S_0(k)}\right\}\right]\\
       &\,\,\,\,\vdots\\
       &=E_0\left[\sum_{k=0}^{\tau-1}A \frac{B_{k,0}(k)}{S_0(k)}D_{k,0}(k)
         e_{R,0}(k)S_0(k)\left\{\frac{\hat S(\tau)}{\hat
         S(k)}-\frac{S_0(\tau)}{S_0(k)}\right\}\right]\\
       &=E_0\left[\sum_{t=1}^\tau -A\frac{\hat S(\tau)}{\hat
         S(t)}S_0(t-1)\{h_0(t)-\hat h(t)\}\left\{\sum_{k=0}^{t-1}\frac{B_{k,0}(k)}{S_0(k)}D_{k,0}(k)
         e_{R,0}(k)\right\}\right]\\
       &=E_0\left[\sum_{t=1}^\tau\sum_{k=0}^{t-1}
         -A\frac{1-R_0}{1-d_{k,0}(0)}S_0(t-1)\{h_0(t)-\hat
         h(t)\}\frac{\hat S(\tau)}{\hat
         S(t)}\frac{B_{k,0}(k)}{S_0(k)}D_{k,0}(k)
         e_{R,0}(k)\right]\\
       &=E_0\left[\sum_{t=1}^\tau\sum_{k=0}^{t-1}
         -A\frac{1-R_0}{1-d_{k,0}(0)}\frac{1-L_1}{1-b_{k,0}(1)}S_0(t-1)\{h_0(t)-\hat
         h(t)\}\frac{\hat S(\tau)}{\hat
         S(t)}\frac{B_{k,0}(k)}{S_0(k)}D_{k,0}(k)
         e_{R,0}(k)\right]\\
       &=E_0\left[\sum_{t=1}^\tau\sum_{k=0}^{t-1}
         -A\frac{I_t}{D_{k,0}(t)B_{k,0}(t-1)}S_0(t-1)\{h_0(t)-\hat
         h(t)\}\frac{\hat S(\tau)}{\hat
         S(t)}\frac{B_{k,0}(k)}{S_0(k)}D_{k,0}(k)
         e_{R,0}(k)\right]\\
       &=E_0\left[\sum_{t=1}^\tau
         -AI_t\sum_{k=0}^{t-1}\left\{\frac{S_0(t-1)}{B_{k,0}(t-1)}\frac{S_0(\tau)}{
         S_0(t)}\frac{B_{k,0}(k)}{S_{0}(k)}\frac{D_{k,0}(k)}{S_{k,0}(t)}
         e_{R,0}(k)\right\}\{L_t-\hat
         h(t)\}\right] .
  \end{align*}
  Similar arguments can be used to show that (\ref{betaA}) equals
  \[-E_0\left[\frac{e_{A,0}}{q_0}A\sum_{t=1}^\tau\frac{S_0(\tau)}{S_0(t)}\frac{S_0(t-1)}{B_{t,0}(t-1)}\frac{I_t}{D_{t,0}(t)}\{L_t
      - \hat h(t)\}\right],\]
  concluding the proof of the theorem.
\end{proof}
\subsection*{Theorem~\ref{theo:dr}}
This result follows from (\ref{eq:wh}) in the main document and
Lemma~\ref{lemma:betaaslin} below.

\subsection*{Other results}


\begin{lemma}\label{lemma:betaaslin}
  Assume \ref{ass:donsker} and \ref{ass:DR2}. Then we have
  \[\beta(\hat\eta)=-(\Pn-P_0)\{D_{A,g_1} + D_{R,g_1} +
    D_{L,h_1}\}+o_P(n^{-1/2}).\]
\end{lemma}
\begin{proof}
  From Theorem~\ref{lemma:betarep} we have
  $\beta(\hat \eta)=P_0\{D_{A,\hat g} + D_{R,\hat g} + D_{L,\hat h}\}
  + o_P(n^{-1/2})$. We will show that
  \begin{equation}
    P_0D_{R,\hat g}=-(\Pn-P_0)D_{R,g_1} + o_P(n^{-1/2}).\label{eq:asR}
  \end{equation} The proof
  for the other components of $\beta(\hat\eta)$ follows analogous
  steps. Assume first that $g_1=g_0$. Denote
  \[\hat D_{R, \hat g}(o) = -\sum_{k=0}^{\tau-1}a\,j_k \hat H_R(k, w)
    \{r_k - \hat g_R(k,w)\}.\] Note that, by construction,
  $\Pn \hat D_{R, \hat g}=0$. Thus we have
  \[P_0D_{R,\hat g} = -(\Pn - P_0)\hat D_{R,\hat g} + P_0(D_{R,\hat g}
    - \hat D_{R,\hat g}),\] where we added and subtracted
  $P_0\hat D_{R,\hat g}$. We have
  \[P_0(D_{R,\hat g} - \hat D_{R,\hat g}) = \int
    \sum_{k=0}^{\tau-1}a\,j_k \{\hat H_R(k, w) -
    H_R(k,w)\}\{g_{R,0}(k,w) - \hat g_R(k, w)\}dP_0(o).\] The
  Cauchy-Schwartz inequality shows
  \[P_0(D_{R,\hat g} - \hat D_{R,\hat g}) = O_P\left(||\hat g_R -
      g_{R,0}||\,||\hat H_R - H_R||\right).\] We will now argue that
  $||\hat H_R - H_R||$ may be decomposed as a sum of two terms: one
  exclusively related to estimation of $g_0$, and one exclusively
  related to the smoothing method used to obtain $\hat\lambda$. Lemma
  \ref{lemma:telescope} along with the Cauchy-Schwartz inequality and
  the definition of $H_R$ show that
  \[||\hat H_R - H_R||=O_P\left(||\hat v_{k} - v_{k,0}|| + ||\hat
      u_{k} - u_{k,0}|| + ||\hat g_R - g_{R,0}|| + ||\hat e_L -
      e_{L,0}||\right)\] Recall the definition of $\hat v_{k,0}$,
  $\hat u_{k,0}$, and $\hat e_{L,0}$ as the corresponding true
  expectations conditional on the estimated covariate $\hat C_g$. The
  triangle inequality shows
  \[||\hat v_{k} - v_{k,0}||\leq ||\hat v_{k,0}-v_{k,0}||+||\hat v_k -\hat
    v_{k,0}||,\] where the fist term in the right hand side
  converges as $||\hat C_g- C_g||$, and the second term is assumed
  $o_P(n^{-1/4})$ (\ref{ass:ass4}). Analogous inequalities hold for
  $\hat u_{k,0}$ and $\hat e_L$. Since
  $||\hat C_g- C_g||=O_P(||\hat g_R - g_{R,0}||+||\hat g_A -
  g_{A,0}||)$, we get
  \[P_0(D_{R,\hat g} - \hat D_{R,\hat g}) = O_P\left(||\hat g_R -
      g_{R,0}||\{||\hat g_R - g_{R,0}||+||\hat g_A - g_{A,0}|| +
      o_P(n^{-1/4})\}\right).\] Under condition~\ref{ass:DR2} this term is
  $o_P(n^{-1/2})$. Under \ref{ass:donsker} and \ref{ass:DR2}, example
  2.10.10 of \cite{vanderVaartWellner96} yields that
  $\hat D_{R,\hat g}$ is in a Donsker class. Thus, according to
  theorem 19.24 of \cite{vanderVaart98}:
  \[P_0D_{R,\hat g} = -(\Pn-P_0)D_{R,g_0} + o_P(n^{-1/2}).\] If
  $h_1=h_0$, then $e_L(t,w)=0$, which implies $H_R(t,w)=0$. Thus,
  $D_{R,\hat g}(o)=D_{R,g_1}(o)=0$, and (\ref{eq:asR}) follows trivially,
  concluding the proof of the lemma.

\end{proof}
\begin{lemma}\label{lemma:telescope}
  For two sequences $a_1,\ldots,a_m$ and $b_1,\ldots,b_m$ such that
  $a_t\neq 1$ and $b_t\neq 1$, we have
  \[\prod_{t=1}^{m}(1-a_t) - \prod_{t=1}^{m}(1-b_t) = \sum_{t=1}^{m}\left\{\prod_{k=1}^{t-1}(1-a_k)(b_t-a_t)\prod_{k=t+1}^{m}(1-b_k)\right\}.\]
\end{lemma}
\begin{proof}
  Replace $(b_t-a_t)$ by $(1 - a_t) - (1 - b_t)$ in the right hand
  side and expand the sum to notice it is a telescoping sum.
\end{proof}



\bibliographystyle{plainnat}
\bibliography{tmle}
\end{document}